\documentclass{article} 
\usepackage[preprint]{iclr2027_conference}
\usepackage{times}

\usepackage[utf8]{inputenc} 
\usepackage[T1]{fontenc}    
\usepackage{hyperref}       
\usepackage{url}            
\usepackage{booktabs}       
\usepackage{amsfonts}       
\usepackage{nicefrac}       
\usepackage{microtype}      
\usepackage{xcolor}         

\usepackage{amsmath}
\usepackage{amsthm}
\usepackage{amssymb}
\usepackage{thmtools}
\usepackage{thm-restate}
\usepackage[capitalize]{cleveref}
\usepackage{todonotes}
\usepackage[normalem]{ulem}
\usepackage{subfig}
\usepackage{enumitem}
\usepackage{tcolorbox}
\usepackage{blindtext}
\usepackage{multicol}

\usepackage{graphicx}
\usepackage{wrapfig}
\usepackage{multirow}

\usepackage{algorithm}
\usepackage{algpseudocode}

\newtheorem{theorem}{Theorem}[section]

\newtheorem{lemma}[theorem]{Lemma}

\theoremstyle{definition}
\newtheorem{definition}{Definition}[section]

\newtcolorbox{mybox}[1]{colback=white,colframe=red!75!black,fonttitle=\bfseries,title=#1}
\newtcolorbox{mybox2}[1]{colback=red!5!white,colframe=red!75!black,fonttitle=\bfseries,title=#1}

\title{On Reliability of Membership Inference\\ Vulnerability Evaluation}

\author{
Joonas Jälkö$^{1}$ \quad
Gauri Pradhan$^{1}$ \quad
Ossi Räisä$^{2}$ \quad
Antti Honkela$^{1}$ \\
$^{1}$University of Helsinki, Finland \\
$^{2}$CISPA Helmholtz Center for Information Security, Germany \\
\texttt{\{joonas.jalko, gauri.pradhan, antti.honkela\}@helsinki.fi} \\
\texttt{ossi.raisa@cispa.de}
}

\newif\ifdoubleblind

\doubleblindtrue   

\doubleblindfalse  

\newcommand{\cA}{\mathcal{A}}
\newcommand{\cS}{\mathcal{S}}

\newcommand{\cM}{\mathcal{M}}

\newcommand{\cL}{\mathcal{L}}
\newcommand{\cP}{\mathcal{P}}

\newcommand{\datatrain}{D_{\textrm{train}}}

\newcommand{\xtarget}{x_{\textrm{target}}}
\newcommand{\modelinx}{\mathcal{M}_{\textrm{in}}}
\newcommand{\modeloutx}{\mathcal{M}_{\textrm{out}}}
\newcommand{\modeltarget}{\mathcal{M}_{\textrm{target}}}
\newcommand{\logittarget}{\mathcal{S}_{\textrm{target}}}

\newcommand{\pinx}{p_{\textrm{in}, x}}
\newcommand{\poutx}{p_{\textrm{out}, x}}
\newcommand{\muin}{\mu_{\textrm{in}}}
\newcommand{\muout}{\mu_{\textrm{out}}}
\newcommand{\muinx}{\mu_{\textrm{in}, x}}
\newcommand{\muoutx}{\mu_{\textrm{out}, x}}
\newcommand{\sigin}{\sigma_{\textrm{in}}}
\newcommand{\sigout}{\sigma_{\textrm{out}}}
\newcommand{\siginx}{\sigma_{\textrm{in}, x}}
\newcommand{\sigoutx}{\sigma_{\textrm{out}, x}}

\newcommand{\tpr}{\mathrm{TPR}}
\newcommand{\fpr}{\mathrm{FPR}}

\newcommand{\fnr}{\mathrm{FNR}}
\newcommand{\auc}{\mathrm{AUC}}

\begin{document}


\maketitle

\begin{abstract}
Membership inference attacks (MIAs) are popular methods for empirically assessing the leakage of sensitive information in the training data through models or statistics learned from the data. The MI vulnerability is often evaluated through a binary classifier that tries to predict whether a particular sample was in the training data. 
In order to evaluate the effectiveness of MIAs multiple \textit{shadow models} are trained using random partitions of a larger dataset. After training the shadow models the MI vulnerability can be evaluated for all the samples for which we obtained shadow models. In order to evaluate the MI vulnerability reliably one needs a lot of shadow models which can be computationally infeasible. Therefore instead of reporting the actual sample level vulnerabilities aggregates over multiple samples are often reported in practice.
We demonstrate two key weaknesses in typical MIA evaluation pipeline.
First, we show that sampling the shadow datasets from a fixed superset leads to finite sample bias inflating the vulnerability estimates.
Second, we show that evaluating the true positive rate (TPR) by concatenating MIA scores across multiple individuals, commonly used in the very low false positive rate (FPR) regime, is not calibrated across the per-sample FPRs.
For both weaknesses we propose fixes that in the most simple approximate form do not incur any additional computation cost. We show that with additional computation one can further improve the reliability of the vulnerability estimation.
\end{abstract}

\section{Introduction}
\label{sec:intro}
Efficiently evaluating the privacy risk associated with a particular machine learning algorithm applied to a particular dataset is critical for calibrating the level of privacy protection, but no good general solutions exist for this problem.
The difficulty of the problem arises from many sources: 1) the privacy risk is a statistical phenomenon and evaluating it requires repeated observations; 2) identifying the most vulnerable samples needed to evaluate maximum risk can be difficult; and 3) it is difficult to know if the applied attack is optimal.

Privacy auditing methods \citep{annamalai2025hitchhikers} address each of these by 1) developing specific one-run-auditing (ORA) techniques \citep{DBLP:conf/nips/0002NJ23,mahloujifar2025auditing}; 2) crafting canary samples to represent maximally vulnerable samples in the entire data universe \citep{boglioni2026optimizing}; and 3) potentially using more easily auditable internals of the algorithm to ensure strong attacks.
These methods have weaknesses: the reliability of ORA is fundamentally limited \citep{xiang2025,keinan2025how}, and the actual samples may be more or less vulnerable than the canaries.

To evaluate the vulnerability of the actual samples, we focus on membership inference attacks (MIAs) for privacy risk estimation.
MIA is based on a binary classification problem: was a given sample part of the training dataset or not.
MIAs are useful for this purpose because unlike stronger privacy attacks like attribute inference and training data reconstruction, the setting is highly standardised.
MIA vulnerability is related to vulnerability to these other attacks by reductions \citep{salem_games_2023}.
The most widely used privacy enhancing technology in machine learning, differential privacy \citep{DBLP:conf/tcc/DworkMNS06}, can be naturally interpreted as a bound on MIA success \citep{wasserman2010}.

Following \citet{carlini_membership_2021}, MIA evaluation has focused on
evaluating the true positive rate ($\tpr$) of membership identification at fixed
low false positive rate ($\fpr$) averaged over multiple members of the training
data. This is most sensitive to the adversary's ability to confidently detect
even a small number of true members of the training dataset.

When we evaluate the MIA success over multiple data points, we are measuring the
average vulnerability instead of the sample-level vulnerability.
\citet{AerniZT24}, \citet{tobaben2025impactdatasetpropertiesmembership} and \citet{Knolle2026} have shown how the average vulnerability can significantly fail to
communicate the vulnerability of the samples. However, reliably estimating the per-sample
vulnerability becomes expensive for
\emph{shadow-model-based} MIA where we need to train multiple models to estimate
the effect the sample has to the learning outcomes.

In this paper, we study both the per-sample and average estimates of MIA vulnerability.
We demonstrate that estimating the per-sample vulnerability by learning
the shadow models by sampling their training data from a finite superset
leads to inflated vulnerability. This bias also affects the average vulnerability 
estimation. We demonstrate how to mitigate the bias with a simple correction method.
We also demonstrate that the reporting the average vulnerability through a 
ROC curve computed from MIA scores combined over multiple target samples is not calibrated
in terms of the per-sample $\fpr$. We demonstrate a simple post-processing fix for this issue.

\textbf{Contributions}
\begin{itemize}[nosep,leftmargin=2em]
    \item We demonstrate that the most common approaches for performing shadow-model-based MIA
    suffer from finite population bias in
    estimating the in/out-distribution variances. This leads to inflated 
    vulnerability estimates (\Cref{sec:finite-population,sec:results}).
    We propose ways to fix this.
    \item We show that concatenated MIA aggregation reports an average over
    per-sample vulnerabilities, but that a single global threshold induces
    heterogeneous per-sample $\fpr$s. This miscalibration can distort
    conclusions about which learning algorithms or defenses are more private
    (\cref{subsec:concat_mia}).
    \item We show that offline LiRA \citep{carlini_membership_2021} avoids the
    calibration problem. Based on this we form a post-processing-based approach
    that calibrates the per-sample $\fpr$s and allows the concatenation to
    report average per-sample $\tpr$ with the targeted $\fpr$ under mild
    Gaussian assumptions
    (\cref{subsec:post-processing,sec:results}).
\end{itemize}

\subsection*{Related Work}
Many works have focused on MIA evaluation through the ROC-curve for the binary
classification task. \citet{carlini_membership_2021} proposed using TPR at low
FPR as a metric for the vulnerability, and proposed two MIA attacks: the online
and offline LiRA attack which are both based on approximating the shadow
distributions with Gaussians. Both of their attacks compare the target point to
reference models and devise a threshold classifier to make the decision whether
the point is in the training data or not. 
\citet{ye_enhanced_2021} also focuses on TPR at low FPR, and shows how the
decision threshold for the classifier should be chosen to match different
types of membership inference games, e.g. how to choose the threshold if we 
want to target the MIA for fixed sample or model. In our work we study the
per-sample vulnerability, which matches their Attack R setting.
While these works focus on vulnerability evaluation against decision threshold
classifiers, neither of the works discuss how the vulnerability should be
aggregated over multiple models or samples. This is crucial in order to
understand the vulnerability over the population reliably, as sample-wise
vulnerability evaluation becomes expensive due to massive amounts of shadow
models needed to achieve statistically significant evaluation.

While \citet{carlini_membership_2021} do not study the aggregation of the MIA vulnerabilities formally, implementations of the LiRA algorithm e.g. in the ML Privacy Meter library \citep{Murakonda2020} compute an aggregate of the vulnerability by concatenating the LiRA scores over multiple individuals and target models.
In many cases, this is the only feasible way to reach very small FPR values without training an excessive number of models.
In this work, we explain mathematically what this strategy actually measures, and propose an alternative that more closely relates to the per-sample vulnerability.

\citet{AerniZT24} showed how the average vulnerability can significantly underestimate
the vulnerability of a single data point. They demonstrated this by computing the 
per-sample vulnerabilities for multiple samples and showed that for a fixed per-sample
$\fpr=0.1\%$ the $\tpr$ can reach $99.9\%$ while the population average remains at
$4\%$.
Later, \citet{tobaben2025impactdatasetpropertiesmembership} and \citet{Knolle2026} have similarly shown that individual points can be significantly more vulnerable than the population average.
Our study on the aggregated vulnerability is closely related, but instead of 
learning the per-sample decision thresholds which is computationally expensive, we study
the average vulnerability of the fixed threshold classifier arising from the concatenated
MIA scores.

\section{Background}

\subsection{Membership inference attacks}
The goal of the MIA is to infer whether a particular training sample belongs
to the training dataset of a particular learning outcome
\citep{shokri_membership_2016}. Denoting the target sample with $\xtarget$
and the learning outcome with $\modeltarget$, MIA can be formally stated
as the following binary hypothesis test 
\begin{equation}
    H_0: \xtarget \notin \datatrain \quad \text{ vs. } \quad H_1: \xtarget \in \datatrain,
\end{equation}
based on observing $\modeltarget \leftarrow \cA(\datatrain)$ where $\datatrain$
is latent. In this work we will focus on \emph{shadow-model}-based MIA where
the test statistic for the above hypothesis test is based on comparing
the observed $\modeltarget$ to the distributions of learning outcomes when
$\xtarget$ is in or out of the $\datatrain$. For a data distribution $\cP$,
the in and out samples are distributed as
\begin{align}
    D &\sim \cP^{N}  \\ 
    \modeloutx &= \cA(D \setminus \{\xtarget\}) \\
    \modelinx &= \cA(D \cup {\xtarget}).
\end{align}
We will denote the in and out distributions for a sample $x$ respectively
as $\pinx$ and $\poutx$. Instead of performing the attack directly on the 
$\modeltarget$, it is common to apply some deterministic function, e.g.
a logit, to the model and target point before the attack to make the distribution 1D and approximately Gaussian. We denote such
function that takes the sample and the learning outcome as arguments 
with $\cS(\xtarget, \mathcal{M})$.

\citet{ye_enhanced_2021} studied MIA under various different threat models, i.e.\ 
with different levels of attacker's knowledge. They proposed a unified threshold
classifier type of attack, where the decision threshold is adapted to the
targeted threat model, e.g. if we want to attack particular target samples or
models. In this paper we will focus on the per-sample attacks, and the MIA
hypothesis becomes 
\begin{align}
    H_0: \, \cS(\xtarget, \cA(\datatrain)) \sim \poutx  \quad
        \textrm{ vs. } \quad H_1: \, \cS(\xtarget, \cA(\datatrain)) \sim \pinx.
\end{align}

In other words, we are studying the distinguishability of $\xtarget$
from the learning outcome over the stochasticity arising from any inherent
randomness of $\cA$ and from the latent training dataset $D \sim \cP^n$
that may or may not include $\xtarget$.

The LiRA algorithm by \citet{carlini_membership_2021} has proved to be a highly
effective. From the trained shadow models, LiRA infers normal approximations for
the in/out-distributions for each target point separately. The LiRA
score for a target point $x$ and model
$\modeltarget$ is given as the log-likelihood ratio (LLR)
\begin{align}
    \cL(x, \modeltarget) := \log \frac{\pinx(\logittarget)}{\poutx(\logittarget)} 
    = \log \frac{\mathcal{N}(\logittarget; \, \muinx, \siginx^2)}{\mathcal{N}(\logittarget; \, \muoutx, \sigoutx^2)},
\end{align}
where the parameters $\muinx, \muoutx, \siginx, \sigoutx$ are
learned from the shadow models and $\logittarget := \mathcal{S}(x, \modeltarget)$.
\citet{carlini_membership_2021} also proposed a computationally cheaper
alternative to LiRA called \emph{offline} LiRA that does not require
learning the IN distribution. In offline LiRA, the membership of the 
observation $\cS_{\mathrm{target}}$ is classified based on the following 
statistic
\begin{align}
    \label{eq:offline-lira}
    \Lambda_{\mathrm{off}}(\cS_{\mathrm{target}}) = \Pr(S^{\mathrm{out}}_x \leq \cS_{\mathrm{target}}) 
                                  = \Phi\left(\frac{\cS_{\mathrm{target}} - \muoutx}{\sigoutx}\right),
\end{align}
where $\Phi$ denotes the CDF of the standard normal distribution.

\subsection{Measuring vulnerability}
Since MIA is a binary hypothesis test, the vulnerability can be measured through
the properties of the distinguisher performing this test.
\citet{carlini_membership_2021} suggested measuring the vulnerability in terms of
the attack's $\tpr$ at low fixed $\fpr$. This metric has been since widely
adopted in the literature
\citep[see, e.g.][]{zarifzadeh_low-cost_2023,liu_membership_2022}. Empirically, the
$\fpr$ and $\tpr$ are commonly obtained from a ROC-curve built from the MIA
classifiers' scores and corresponding ground truth labels.

\paragraph{Trade-off function.}
The $\fpr$ is an empirical estimate of a classifier's type I error, while the
$\tpr=1 - \fnr$ with $\fnr$ being the estimate of the type II
error. \emph{Trade-off function} is an important tool that allows to understand
these two errors for arbitrary classifier. It gives a lower bound for any
classifiers $(\fpr, \fnr)$ curve, i.e. the inverse ROC-curve, and is 
defined as follows

\begin{definition}[\cite{gaussianDP}]
    For any distributions $P$ and $Q$ on a measurable space $(\mathcal{X},\mathcal{F})$,
    the trade-off function $T(P, Q): [0,1] \rightarrow [0,1]$ is given as
    \begin{align}
        T(P, Q)(\alpha) = \inf_\phi \{ \beta_\phi \, : \, \alpha_\phi \leq \alpha \}
    \end{align}
    where $\phi: \mathcal{X} \rightarrow [0,1]$ is a rejection rule with
    \begin{align}
    \alpha_\phi = \mathbb{E}_P[\phi], \; \beta_\phi = 1 - \mathbb{E}_Q[\phi].
    \end{align}
\end{definition}

The next Lemma shows that the trade-off function is invariant to bijective 
post-processing.
\begin{restatable}{lemma}{lembijinv}
    \label{lem:bijective-invariance}
    Let $P$ and $Q$ be distributions on $\mathbb{R}$, and let
    $f(x)$ be a bijective function.
    Denote by $P^{f}$ and $Q^{f}$ the pushforward distributions of $P$ and $Q$
    under $f$, i.e., $P^{f}(B) = P(f^{-1}(B))$ for all Borel sets
    $B \subseteq \mathbb{R}$, and similarly for $Q^{f}$. Then, for all
    $\alpha \in [0,1]$,
    \[
        T(P, Q)(\alpha)
        \;=\;
        T\bigl(P^{f}, Q^{f}\bigr)(\alpha).
    \]
\end{restatable}
The proof for this Lemma can be found in \cref{sec:app-proofs} of the Appendix.

\paragraph{Measuring vulnerability using a parametric model}
Recently \citet{Knolle2026} proposed using the area under the curve (AUC)
of the MI classifier as the vulnerability measure. They showed that
if the in/out distributions are Gaussians as $\mathcal{N}(\muin, \sigin^2)$
and $\mathcal{N}(\muout, \sigout^2)$, then a threshold classifier applied
directly on the in/out samples
will have an AUC of
\begin{align}
    \label{eq:knolle-auc}
    \auc = \Phi\left(
        \frac{\muin-\muout}{\sqrt{\sigin^2 + \sigout^2}}
    \right).
\end{align}
In \cref{sec:app-proofs} of the Appendix we show similar analytical forms for
the LLR test and the ROC curve. Note that while the analytical vulnerability
estimates such as the AUC above are convenient to compute based on the in/out
statistics, they can be biased due to the parametric model used for the shadow
distributions.

\section{Finite Population Bias in MIA}
\label{sec:finite-population}

When implementing a MIA in practice on a real dataset, we do not have access to
the underlying data generating process $\cP$ but only a finite superset
$D_{\mathrm{full}} = \{x_1,\dots,x_{N_{+}}\} \sim \cP^{N^+}$ from which the
training sets of each model must be sampled. For each target model $\cM$, the
corresponding training set $D$ of $N$ elements is sampled without replacement
from $D_{\mathrm{full}}$.
This approach comes with two important differences compared to the ``ideal''
LiRA setting where shadow datasets are independently drawn as $D \sim \cP^N$:
\begin{enumerate}[nosep,leftmargin=2em]
    \item All shadow and target models are trained on overlapping subsets of the
          \emph{same} finite $D_{\mathrm{full}}$.
    \item Each $D$ is sampled without replacement from $D_{\mathrm{full}}$,
          rather than from an infinite population.
\end{enumerate}
As a consequence, estimators of the in/out variances of the learning outcomes
(or LiRA scores) are subject to a finite population bias. Intuitively,
reusing the same finite set $D_{\mathrm{full}}$ multiple times makes the
empirical variability of the shadow models smaller than what one would observe
if each training set were drawn independently from $\mathcal{P}$.

\subsection{Finite Population Correction}
\label{subsec:FPC}
To see the effect at a high level, consider any scalar statistic $T(D)$ of a
dataset $D$ of size $N$ drawn \emph{without replacement} from a finite population
of size $N_{+}$. Classical sampling theory \citep{cochran1977sampling} shows that
\begin{equation}
    \mathrm{Var}[T(D)]
    \approx
    \mathrm{Var}_{\text{iid}}[T(D)]
    \cdot
    \Bigl(1 - \frac{N}{N_{+}}\Bigr),
    \label{eq:fpc}
\end{equation}
where $\mathrm{Var}_{\text{iid}}[T(D)]$ is the variance that would be obtained if
the elements of $D$ were sampled independently from the underlying population
(i.e., with replacement). The factor $\mathrm{FPC}  = 1 - \nicefrac{N}{N_+}$
is the standard finite population correction. When $N / N_{+}$ is small, the
correction is negligible, but for larger $N$ the effect becomes significant.
When training shadow models for MIA, it is common to choose $N$ to be a large
fraction of $N_{+}$ (e.g.\ $N = 0.5\,N_{+}$ in the Privacy meter library
\citep{Murakonda2020}), in which case $\mathrm{FPC} \approx 0.5$ and the variance
may be substantially underestimated.

In our setting, the statistic is the learning outcome (or its LiRA score) for a
fixed point $x$ across shadow models. We estimate the in/out distributions for
$x$ from $\{\cM^{(m)}\}_{m=1}^M$, but all $\cM^{(m)}$ are trained on overlapping
subsets of a fixed finite $D_{\mathrm{full}}$. 
Therefore the empirical in/out variances estimated from the shadow
models are biased towards \emph{smaller} values compared to the variances one
would obtain from independent draws $D \sim \mathcal{P}^N$.

Since LiRA's distinguishability is driven in large part by the separation
between the in and out distributions relative to their variances, underestimating
these variances can lead to inflated estimates of membership inference
vulnerability. The effect is easy to see from \cref{eq:knolle-auc}. Consider both
the estimated in and out stds $\hat{\sigma}_{\text{in}}$, $\hat{\sigma}_{\text{out}}$ 
are equally effected by the bias, i.e. $\hat{\sigma}_{\text{in}} = \sqrt{f}\sigin$
and $\hat{\sigma}_{\text{out}} = \sqrt{f} \sigout$ where $f = \mathrm{FPC} \leq
1$. We get that
\begin{align}
    \widehat{\auc} = \Phi\left(\frac{1}{\sqrt{f}}\frac{\muin - \muout}{\sqrt{\sigin^2 + \sigout^2}}\right)
    > 
    \Phi\left(\frac{\muin - \muout}{\sqrt{\sigin^2 + \sigout^2}}\right) = \auc.
\end{align}

The ROC-curve-based vulnerability evaluations are also affected
by the finite population bias. However, 
in addition to the FPC fixes to the IN/OUT distributions
that correct the LiRA score computation, we also need to correct the bias
on the sampled target models that we use to compute the ROC curve.
A simple way to correct the test set is to rescale each observed 
target model by the FPC factor within its true class as
\begin{align}
    g(\cS, y) = \mu_y + \frac{\cS - \mu_y}{\sqrt{f}},
\end{align}
where the $y \in \{\text{in}, \text{out}\}$. This maps the test points
to the population distribution. Note that the map depends on knowing
the label $y$, and therefore cannot be considered as part of an
actual attack. However, we are simply using the transformation in the
evaluation stage to remove the bias from the test points, not to 
implement an actual attack.

\subsection{Variance Decomposition}
\label{sec:variance-decom}
Above we demonstrated the bias on vulnerability estimation by assuming the
variance is affected as $\sigma_{\mathrm{obs}}^2 = f \sigma_{\mathrm{pop}}^2$
which we will call \emph{Simple} FPC.
However, there might be population-independent sources of randomness in the
in/out distributions arising e.g.\ from the learning algorithm. These sources of
variation are not subject to the finite population correction, and therefore
correcting the observed total variance by dividing with $f$ might over amplify
some parts of the variance.

In order to separate the population-independent sources of randomness,
we will use the following simple variance decomposition for both the in and
out distributions
\begin{equation}
    \label{eq:variance-decom}
    \sigma_{\mathrm{obs}}^2(f) = f\sigma^2_{\mathrm{pop}} + \sigma^2_{\mathrm{rem}},
\end{equation}
where \(\sigma^2_{\mathrm{pop}}\) is the component affected by finite-population
sampling and \(\sigma^2_{\mathrm{rem}}\) plays the role of the other
variability. We will call this variance decomposition based correction
the \emph{Decomposed} FPC.

\paragraph{Estimating the decomposition}
We can estimate the $\sigma_\mathrm{pop}$ and $\sigma_{\mathrm{rem}}$ by
computing the observed variance arising from multiple different ratios $N /
N_+$, leading to different FPC factors $f$, and fitting the observations to the
line. Note that computing each observed variance requires multiple shadow models
to be trained. Furthermore with small $N/N_+$, the probability of $\xtarget$
belonging in the shadow dataset is small, and therefore we need a lot of
repeats to estimate the IN variance. Conversely, with large $N/N_+$ we need a lot of
repeats to estimate the OUT variance for each $\xtarget$. In \Cref{alg:decomposed-fpc}, we show how to estimate the corrected variance from 2 sampling ratios $(r_a,r_b)$ and use them to correct the observed IN/OUT variances. In \Cref{fig:vd_heatmap_adult,fig:vd_heatmap_patch}, we compare two-ratio decompositions against all-ratio estimates and select the pair minimizing mean parameter RMSE across 3 repeats. 

\section{Reporting aggregates of the vulnerability}
\label{sec:vulnerability_eval}

Recall that our goal is to estimate the vulnerability of the samples
$\xtarget$. In order to empirically assess this vulnerability, we learn multiple
shadow models trained with or without $\xtarget$ and build the ROC-curve with
the corresponding scores and labels. The in/out scores can be written as
\begin{align}
    & \ell^{\text{(out)}}_{\xtarget} := \cL(\xtarget, \cA(D)), \quad
    \ell^{\text{(in)}}_{\xtarget} := \cL(\xtarget, \cA(D \cup \xtarget)),
\end{align}
where $D \sim \cP^n$.
The $\fpr$ and $\tpr$ for a threshold classifier with a decision threshold $\tau
\in \mathbb{R}$ become
\begin{align}
    & \fpr_x(\tau) = \Pr(\ell^{\text{(out)}}_x > \tau) 
        = \mathbb{E}_{D, \xi}[\mathbf{1}[\ell^{\text{(out)}}_x > \tau]] \\
    & \tpr_x(\tau) = \Pr(\ell^{\text{(in)}}_x > \tau) 
        = \mathbb{E}_{D, \xi}[\mathbf{1}[\ell^{\text{(in)}}_x > \tau]]
\end{align}
where the $D \sim P^n$ and $\xi$ denotes any additional stochasticity in $\cA$.

\subsection{Concatenating the MIA scores}
\label{subsec:concat_mia}

\begin{wrapfigure}{r}{0.42\textwidth}
\centering
\vspace{-0.5em}
\includegraphics[width=\linewidth]{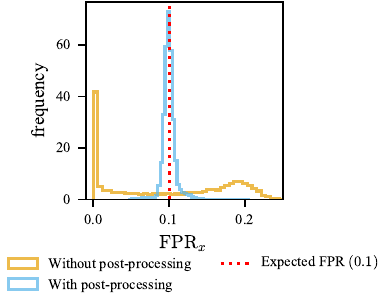}
\caption{\looseness=-1 \textbf{Per-sample FPR distributions under concatenation-based MIA evaluation.} Naive concatenation yields heterogeneous FPRs, while post-processing calibrates samplewise OUT distributions before aggregation, reducing FPR variation. Results are shown for TabPFN models trained on Adult dataset.}
\label{fig:fpr_dist_tabpfn_plot}
\vspace{-1.25\baselineskip}
\end{wrapfigure}
Consider now that we have MIA scores evaluated over $M$ target models and $N$
target data points. Each of the $N$ samples $\xtarget$ was either in our out of
the training data for each of the $M$ models. 
Since reliably estimating per-sample TPR at low FPR requires many target models,
practical evaluations often aggregate scores over samples and models.

Concatenating MIA scores over multiple target models has been used in the past
to learn aggregates of the vulnerability
\citep{carlini_membership_2021,Murakonda2020}. 
In concatenated MIA evaluation, we are merging the MIA scores computed for the
$N$ individuals over $M$ target models into a single vector. The vulnerability
is then evaluated based on a ROC-curve built on these raveled scores and their
corresponding labels.

The $\fpr$ and $\tpr$ for the concatenated approach are given as
\begin{align}
    & \fpr = \Pr(\ell^{\text{(out)}} > \tau), \quad  \tpr = \Pr(\ell^{\text{(in)}} > \tau),
\end{align}
where the dependence of both the model and sample to the scores is marginalized
away. 
The next proposition demonstrates how the concatenated approach is related to
the per-sample vulnerability.
\begin{restatable}{proposition}{propmeantprfpr}
    For any fixed threshold $\tau$, the concatenated FPR and TPR are equal to the
    expectation of per-sample FPR$_x$ and TPR$_x$:
    \[
    \mathrm{FPR}(\tau) = \mathbb{E}_x[\mathrm{FPR}_x(\tau)], \qquad
    \mathrm{TPR}(\tau) = \mathbb{E}_x[\mathrm{TPR}_x(\tau)].
    \]
\end{restatable}
Proof can be found in \cref{sec:app-proofs} of the Appendix.
Thus, choosing a single decision threshold $\tau$ that yields $\fpr=\alpha$ at
the concatenated level corresponds to choosing $\tau$ such that the average
per-sample $\fpr_x$ equals $\alpha$, while individual $\fpr_x$ may vary widely.
Hence calibrating the concatenated $\fpr$ does not necessarily serve as a good
proxy of understanding the vulnerability.


\subsection{Post-processing for better aggregation}

\label{subsec:post-processing}

We start by showing that the offline LiRA attack, originally proposed by
\citet{carlini_membership_2021} as a computationally more feasible version of
LiRA, avoids the FPR calibration problem. 
For a fixed sample $x$, suppose the learning outcome $\cS(x, A(D))$ satisfies
\begin{align}
    S^{\mathrm{out}}_x \sim \mathcal{N}(\muoutx, \sigoutx^2),
    \qquad
    S^{\mathrm{in}}_x \sim \mathcal{N}(\muinx, \siginx^2),
\end{align}
and define the mean gap $\Delta_x = \muinx - \muoutx$.
The transformation $\Lambda_{\mathrm{off}}$ in \cref{eq:offline-lira}
is a bijection since it is a composition of the bijection $\Phi$ and 
an affine mapping $f_x(t) = (t - \muoutx) / \sigoutx$. 
Therefore \cref{lem:bijective-invariance} tells us that 
\begin{align}
\label{eq:tradeoff-gaussian-offline-lira}
    T\!\left(
        S^{\mathrm{out}}_x, S^{\mathrm{in}}_x
    \right)
    &=T\!\left(
        \Lambda_{\mathrm{off}}(S^{\mathrm{out}}_x), 
        \Lambda_{\mathrm{off}}(S^{\mathrm{in}}_x)
    \right)
    =
    T\!\left(
        \mathcal{N}\!\left(\frac{\Delta_x}{\sigoutx}, \frac{\siginx^2}{\sigoutx^2}\right),
        \mathcal{N}(0,1)
    \right).
\end{align}

We can further make this transformation independent of the ordering of
in/out distributions by multiplying with $\mathrm{sign}(\Delta_x)$
\begin{align}
    \label{eq:aff-trans}
    f_x(t) =
    \frac{\mathrm{sign}(\Delta_x)(t - \muoutx)} {\sigoutx}.
\end{align}
By Lemma~\ref{lem:bijective-invariance}, we get
\begin{equation}
\label{eq:tradeoff-gaussian-standardized}
    T\!\left(
        \mathcal{N}(\muinx,\siginx^2),
        \mathcal{N}(\muoutx,\sigoutx^2)
    \right)
    =
    T\!\left(
        \mathcal{N}\!\left(\frac{|\Delta_x|}{\sigoutx}, \frac{\siginx^2}{\sigoutx^2}\right),
        \mathcal{N}(0,1)
    \right).
\end{equation}

We can see that both of the transformations normalize all the
out scores into $N(0, 1)$. Therefore no matter which decision rule
we apply to the out scores, we will get the same FPR across all the samples
and as a result the average gets calibrated.
\Cref{fig:fpr_dist_tabpfn_plot}
shows this in practice: naive concatenation leads to large spread in $\fpr$
values while post-processing concentrates them on the target $\fpr$.


When $\siginx=\sigoutx$,  we can see from
\eqref{eq:tradeoff-gaussian-standardized} that all per-sample in/out
pairs can be viewed as testing between a standard normal and a shifted
unit-variance normal with mean $|\Delta_x|/\sigma_x$. In this case the 
LLR is monotonic, and therefore we can perform a Neyman-Pearson optimal test
directly on the post-processed samples. When $\siginx \neq \sigoutx$, the 
LLR is no longer a  bijection and a threshold test on the post-processed
samples is no longer Neyman-Pearson optimal. However the $\fpr$ will still
remain calibrated.

Finally, we can estimate the average $\tpr$ as
\begin{align}
    \overline{\tpr}(\tau) = \frac{1}{NM}\sum_{\substack{i=1\ldots N \\ m = 1\ldots M}}
        \mathbf{1}[\hat{\cS}^{\mathrm{in}}_{i,j} > \tau],
\end{align}
where $\hat{\cS}^{\mathrm{in}}$ denotes the post-processed in samples.
We want to emphasize that $\overline{\tpr}$ is not suitable for assessing the
worst-case per-sample nor the per-model vulnerability, but should be seen as a
measure of how much the learning algorithm leaks about the individuals on
average. However, as we are averaging over $NM$ values in total, we can actually
have a rather confident estimate of the underlying average per-sample
vulnerability, even in the low $\fpr$ regime. 


\section{Results}
\label{sec:results}

Next, we test how our theoretical findings and the proposed fixes work in practice.
We experiment on both tabular Adult \citep{BeckerK96} dataset and computer vision (CIFAR10 \citep{krizhevsky2009learning}
and PatchCamelyon \citep{DBLP:conf/miccai/VeelingLWCW18}) datasets. For the tabular data we use the TabPFN
as the target model architecture and for image classification we fine-tune ResNet50 \citep{kolesnikov_big_2019}
and ViT-B-16 \citep{dosovitskiy_image_2020} models. A detailed experimental setup can be found in the 
Appendix (\cref{subsec:experiment_setup}).
\ifdoubleblind
\else 
Source code for our experiments is available at \url{https://github.com/TrustworthyMLHelsinki/mi_vulnerbility_evaluation}. 
\fi

\begin{figure}[htbp]
    \centering
\includegraphics[width=\linewidth]{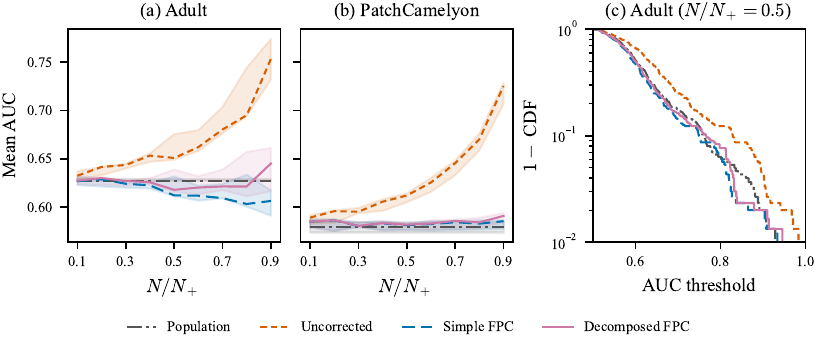}
    \caption{\looseness=-1 \textbf{Finite population bias in MIA.} (a–b) Median of per-seed mean AUCs across three repeats, with minimum–maximum ranges, for PatchCamelyon + ViT-B-16 and Adult + TabPFN. (c) Empirical survival functions of per-target AUCs pooled across 3 repeats for Adult at \(N/N_+=0.5\). Uncorrected estimates increasingly overstate mean AUC as the $N/N_+$ grows. Both corrections reduce this discrepancy; on Adult, simple FPC tends to underestimate mean AUC at larger fractions, while decomposed FPC, using variance decomposition with two ratios, $(r_a,r_b)$, generally remains closer to the population reference. Survival function plots for additional $N/N_+$ are in \Cref{fig:tabpfn_survival_plots,fig:adam_survival_plots}}
    \label{fig:mean-vul-over-f}
    \vspace{-1.5\baselineskip}
\end{figure}

\subsection{Finite Population Correction}
\label{subsec: results_fpc}

We evaluate how sampling training sets from a finite data pool affects MIA vulnerability evaluation. In our experiments, we fix the training set size to $N$ while varying the size of the finite superset $N_+ (= |D_{\mathrm{full}}|)$. This varies the sampling fraction $N/N_+$ without changing the number of samples used to train each model. 
For the population-level MIA estimate, we train $M$ shadow models for each sample using the leave-one-out (LOO) approach. In the LOO setting, for each target $x$, each of the IN models contain $x$ together with $N-1$ samples chosen without replacement from the population dataset, $D_{\mathrm{pop}}$. For the OUT models, we replace the $x$ with another sample from $D_{\mathrm{pop}}$.  

\looseness=-1 In \Cref{fig:mean-vul-over-f}, we evaluate MIA AUC for 100 target samples for (a) TabPFN  models fitted to the Adult dataset using 1000 examples, (b) ViT-B-16 models fine-tuned on PatchCamelyon dataset using 500 examples/class. We observe that as $N/N+$ increases, the uncorrected mean AUC consistently overstates MIA vulnerability w.r.t. the population-reference estimate. For the image classification task on PatchCamelyon, we observe that the FPC (with or without variance decomposition) helps recover the population-level MIA vulnerability. On the Adult dataset, Simple FPC tends to underestimate mean AUC at larger fractions, whereas decomposed FPC generally stays closer to the population reference.

We also test the validity of our variance decomposition of \cref{sec:variance-decom}.
\Cref{fig:var_decomposition_demo_plot} shows the linear fits for the IN variances of 
20 samples randomly sampled from the pool of 100 evaluation samples from Adult and PatchCamelyon datasets, as well as the $R^2$ of the fits across the $100$
targeted samples and $N/N_+\in\{0.1,\ldots,0.9\}$. We see that the linear fits capture large proportion of the observed 
variance for a large proportion of the samples. Fits for OUT variance decomposition are in \Cref{fig:pc_var_decomposition_out}.

\begin{figure}[htbp]
    \centering
    \includegraphics[width=\linewidth]{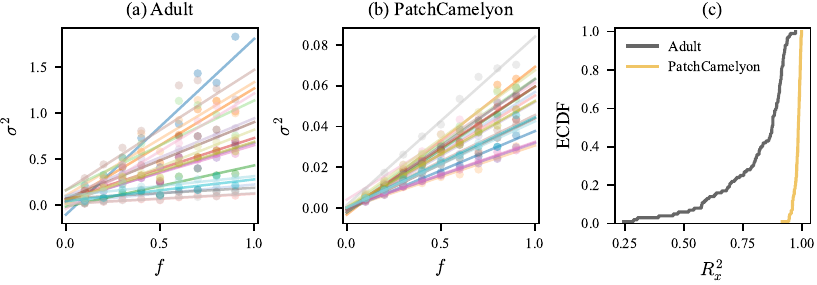}
    \caption{\textbf{Variance decomposition for Adult and PatchCamelyon datasets for IN distribution variances} (a-b) Observed variances and OLS fits $\sigma_x^2(f)=f\sigma_{\mathrm{pop},x}^2+\sigma_{\mathrm{rem},x}^2$ for 20 randomly selected target samples as computed per \cref{sec:variance-decom} for Adult and PatchCamelyon datasets.(c) Empirical CDF of per-target $R_x^2$ across all 100 evaluation targets using observed variances for all ratios $N/N_+\in\{0.1,\ldots,0.9\}$ for both the datasets. We train 2048 models per $N/N_+$, with $N=1000$, and use the collected statistics to estimate each target’s IN variance.}
    \label{fig:var_decomposition_demo_plot}
    \vspace{-1.5\baselineskip}    
\end{figure}

\subsection{Effect of Post-Processing On MIA Evaluation}
\label{subsec: results_pp}

Next, we will focus on the average per-sample vulnerability estimation and
the post-processing strategy of \ref{subsec:post-processing}. \Cref{fig:lira_pp_fpc} 
compares three strategies: (1) the \textbf{Average $\boldsymbol{\tpr}$/Sample} 
which computes the actual average per-sample using $M=4096$ shadow models,
(2) the \textbf{Concatenated (Naive)} which concatenates the plain LiRA scores
for the ROC computation and (3) \textbf{Concatenated (PP)} where we apply the 
post-processing of \cref{eq:aff-trans} to the logits.
We observe that across all settings, the naive concatenation overestimates the average vulnerability whereas the post-processed $\tpr$ closely follows the average per-sample vulnerability even with a moderate ($M=64$) number of shadow models. These results demonstrate that the post-processed approach can be a computationally economical proxy for understanding the average vulnerability. After applying the FPC to remove the finite population bias from the experiments, the same trends persist, but in general with a reduced vulnerability. 

\begin{figure}[htbp]
    \centering
    \includegraphics[width=\linewidth]{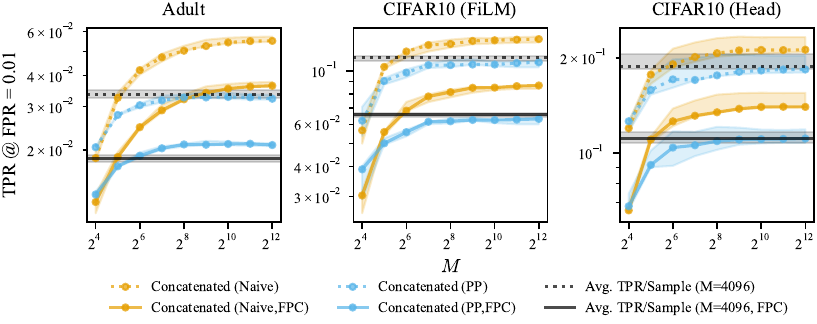}
    \caption{\textbf{Effect of finite population correction and post-processing on MIA performance.} Assessing $\tpr$ at $\fpr=0.01$ against the number of reference models $M$, for Adult + TabPFN and CIFAR10 + FiLM/Head experimental settings. Plots show median $\tpr$ over 3 seeds, with minimum–maximum ranges; We use average per-sample $\fpr$ at $M=4096$ as our baseline. As $M$ grows, naive concatenation exceeds the corresponding per-sample reference, while post-processing (\cref{subsec:post-processing}) brings estimates closer. FPC lowers the estimated TPR with/without post-processing.}
    \label{fig:lira_pp_fpc}
    \vspace{-1.25\baselineskip}    
\end{figure}

\section{Conclusion}

\looseness=-1 Our results show how to evaluate both the per-sample as well as the average MIA vulnerability more reliably.
As pointed out by \citet{AerniZT24}, vulnerability evaluation should focus on the most vulnerable data points, which can be watered down in the average.
This means that the average TPR values need to be interpreted with care.
However with TPR at low FPR as a metric for the vulnerability, the search for
the most vulnerable sample would become computationally infeasible, as
reliable estimation would require a large number of shadow models.
As an alternative one could consider metrics such as the parametric per-example AUC
suggested by \citet{Knolle2026}. While the AUC might not be able to separate
low probability privacy leaks, it can capture the overall worst-case individual privacy risk well
which still is a very important privacy measure. However, while the parametric
AUC estimation is likely more stable than the TPR at low FPR, its reliability
is still affected by the number of shadow models. Studying this is an
interesting direction for future work.
\begin{mybox}{Recommendations}
  \begin{itemize}[leftmargin=*]
    \item For any shadow-model-based MIA vulnerability analysis, apply at least simple FPC
    \item When computing average vulnerability by concatenation, apply post-processing to the concatenated scores
    \item If compute allows, repeat experiment with multiple $N/N_{+}$ for decomposed FPC
  \end{itemize}
\end{mybox}

\ifdoubleblind
\else 
\section*{Acknowledgments}
This work was supported by the Research Council of Finland (Finnish Center for Artificial Intelligence, FCAI, Grant 356499 and Grant 359111), the Strategic Research Council at the Research Council of Finland (Grant 358247) as well as the European Union (Project 101070617). Views and opinions expressed are however those of the author(s) only and do not necessarily reflect those of the European Union or the European Commission. Neither the European Union nor the granting authority can be held responsible for them. This work has been performed using resources provided by the CSC– IT Center for Science, Finland (Project 2003275). We also acknowledge CSC for access to the LUMI supercomputer, owned by the EuroHPC Joint Undertaking, hosted by CSC (Finland) and the LUMI consortium. The authors acknowledge the research environment provided by ELLIS Institute Finland. 
\fi


\bibliography{references}
\bibliographystyle{abbrvnat}


\appendix

\counterwithin*{figure}{part}
\stepcounter{part}
\renewcommand{\thefigure}{A\arabic{figure}}
\setcounter{figure}{0}
\counterwithin*{table}{part}
\stepcounter{part}
\renewcommand{\thetable}{A\arabic{table}}
\setcounter{table}{0}

\section{Appendix}

\subsection{Missing proofs}
\label{sec:app-proofs}
\propmeantprfpr*
\begin{proof}
For the concatenated $\fpr$ we have
\begin{align}
    \fpr = \Pr(\ell^{\text{(out)}} > \tau) 
        &= \mathbb{E}_{x, D, \xi}[\mathbf{1}[\ell^{\text{(out)}}_x > \tau]]
        = \mathbb{E}_x[\mathbb{E}_{D, \xi}[\mathbf{1}[\ell^{\text{(out)}}_x > \tau]]]
        = \mathbb{E}_x[\fpr_x],
\end{align}
which similarly follows for the $\tpr$. 
\end{proof}

\begin{lemma}[\citealp{gaussianDP}, Lemma 1]
    \label{lem:to-post-process}
    For any distributions $P$ and $Q$ and a function $f$, for any $\alpha \in [0,1]$ we have
    \begin{align}
        T(P^{f}, Q^{f})(\alpha) \geq  T(P, Q)(\alpha).
    \end{align}
\end{lemma}

\lembijinv
\begin{proof}
    According to \cref{lem:to-post-process}, for any functions $f$ and $g$, for any $\alpha \in [0,1]$ we have
    \begin{align}
        T(P, Q)(\alpha) \leq T(P^f, Q^f)(\alpha) \leq T(P^{g \circ f}, Q^{g \circ f})(\alpha),
    \end{align}
    where $(g \circ f)(x) = g(f(x))$. Now, if $f$ is a bijection it is invertible and 
    we can chooce $g = f^{-1}$, which yields
    \begin{align}
        &T(P, Q)(\alpha) \leq T(P^f, Q^f)(\alpha) \leq T(P^{f^{-1} \circ f}, Q^{f^{-1} \circ f})(\alpha) = T(P, Q)(\alpha) \\
        \Rightarrow \quad
        &T(P, Q)(\alpha) = T(P^f, Q^f)(\alpha).
    \end{align}
\end{proof}

\paragraph{Analytical Estimation of TPR at Select FPR Using LLRs}

Following \cite{Buss2026},
\[
\operatorname{TPR}_{\mathrm{LLR}}(\alpha)=
\begin{cases}
1-F_{\mathrm{in}}\!\left(
\dfrac{\sigma_{\mathrm{out}}^2}{\sigma_{\mathrm{in}}^2}
F_{\mathrm{out}}^{-1}(1-\alpha)\right),
& \sigma_{\mathrm{in}}^2>\sigma_{\mathrm{out}}^2,\\[8pt]
F_{\mathrm{in}}\!\left(
\dfrac{\sigma_{\mathrm{out}}^2}{\sigma_{\mathrm{in}}^2}
F_{\mathrm{out}}^{-1}(\alpha)\right),
& \sigma_{\mathrm{in}}^2<\sigma_{\mathrm{out}}^2,\\[8pt]
\Phi\!\left(
\Phi^{-1}(\alpha)
+\dfrac{|\mu_{\mathrm{in}}-\mu_{\mathrm{out}}|}
{\sigma_{\mathrm{in}}}\right),
& \sigma_{\mathrm{in}}^2=\sigma_{\mathrm{out}}^2.
\end{cases}
\]
Here $\Phi$ is the standard normal CDF and, for unequal variances,
$F_j$ is the CDF of $\chi_1^2((\mu_j-m)^2/\sigma_j^2)$, where
\[
m=\frac{\mu_{\mathrm{out}}\sigma_{\mathrm{in}}^2
-\mu_{\mathrm{in}}\sigma_{\mathrm{out}}^2}
{\sigma_{\mathrm{in}}^2-\sigma_{\mathrm{out}}^2},
\qquad j\in\{\mathrm{in},\mathrm{out}\}.
\]

\paragraph{Analytical Estimation of AUC Using LLRs}

Let $X\sim\mathcal N(\mu_{\mathrm{in}},\sigma_{\mathrm{in}}^2)$ and
$Y\sim\mathcal N(\mu_{\mathrm{out}},\sigma_{\mathrm{out}}^2)$ be independent,
with positive variances. Then,
\[
\ell(x)=\log\frac{p_{\mathrm{in}}(x)}{p_{\mathrm{out}}(x)}.
\]

\[\operatorname{AUC}_{\mathrm{LLR}}=\Pr(\ell(X)>\ell(Y)).
\]
Expanding the Gaussian log densities gives
\[
\ell(X)-\ell(Y)=\frac{UV}{2\sigma_{\mathrm{in}}^2\sigma_{\mathrm{out}}^2},
\]
where
\[
U=X-Y,\qquad
V=(\sigma_{\mathrm{in}}^2-\sigma_{\mathrm{out}}^2)(X+Y)
  +2(\mu_{\mathrm{in}}\sigma_{\mathrm{out}}^2
      -\mu_{\mathrm{out}}\sigma_{\mathrm{in}}^2).
\]
These variables are jointly Gaussian, with
\[
\begin{aligned}
\mathbb E[U]&=\mu_{\mathrm{in}}-\mu_{\mathrm{out}},&
\operatorname{Var}(U)&=\sigma_{\mathrm{in}}^2+\sigma_{\mathrm{out}}^2,\\
\mathbb E[V]&=(\mu_{\mathrm{in}}-\mu_{\mathrm{out}})
               (\sigma_{\mathrm{in}}^2+\sigma_{\mathrm{out}}^2),&
\operatorname{Var}(V)&=(\sigma_{\mathrm{in}}^2-\sigma_{\mathrm{out}}^2)^2
                      (\sigma_{\mathrm{in}}^2+\sigma_{\mathrm{out}}^2),\\
\operatorname{Cov}(U,V)&=(\sigma_{\mathrm{in}}^2-\sigma_{\mathrm{out}}^2)^2.
\end{aligned}
\]
Thus, for $\sigma_{\mathrm{in}}^2\ne\sigma_{\mathrm{out}}^2$,

\[
\begin{aligned}
\operatorname{AUC}_{\mathrm{LLR}}
&=\Pr(UV>0)\\
&=1-\Pr(U\le0)-\Pr(V\le0)+2\Pr(U\le0,V\le0)\\
&=1-\Phi\!\left(
-\frac{\mu_{\mathrm{in}}-\mu_{\mathrm{out}}}
{\sqrt{\sigma_{\mathrm{in}}^2+\sigma_{\mathrm{out}}^2}}
\right)
-\Phi\!\left(
-\frac{(\mu_{\mathrm{in}}-\mu_{\mathrm{out}})
\sqrt{\sigma_{\mathrm{in}}^2+\sigma_{\mathrm{out}}^2}}
{|\sigma_{\mathrm{in}}^2-\sigma_{\mathrm{out}}^2|}
\right)\\
&\quad+2\Phi_2\!\left(
-\frac{\mu_{\mathrm{in}}-\mu_{\mathrm{out}}}
{\sqrt{\sigma_{\mathrm{in}}^2+\sigma_{\mathrm{out}}^2}},
-\frac{(\mu_{\mathrm{in}}-\mu_{\mathrm{out}})
\sqrt{\sigma_{\mathrm{in}}^2+\sigma_{\mathrm{out}}^2}}
{|\sigma_{\mathrm{in}}^2-\sigma_{\mathrm{out}}^2|};
\frac{|\sigma_{\mathrm{in}}^2-\sigma_{\mathrm{out}}^2|}
{\sigma_{\mathrm{in}}^2+\sigma_{\mathrm{out}}^2}
\right),
\end{aligned}
\]
where $\Phi$ and $\Phi_2(\cdot,\cdot;\rho)$ are the standard univariate
and bivariate normal CDFs, respectively. Gaussianity is assumed for
$X,Y$, not for their LLR scores.

For equal variances, the LLR becomes equivalent to AUC as formulated by \citep{Knolle2026},
\[
\operatorname{AUC}_{\mathrm{LLR}}
=\Phi\!\left(
\frac{|\mu_{\mathrm{in}}-\mu_{\mathrm{out}}|}
{\sqrt{\sigma_{\mathrm{in}}^2+\sigma_{\mathrm{out}}^2}}
\right).
\]

\subsection{Additional Plots}

\begin{figure}[htbp]
    \centering
    \includegraphics[width=\linewidth]{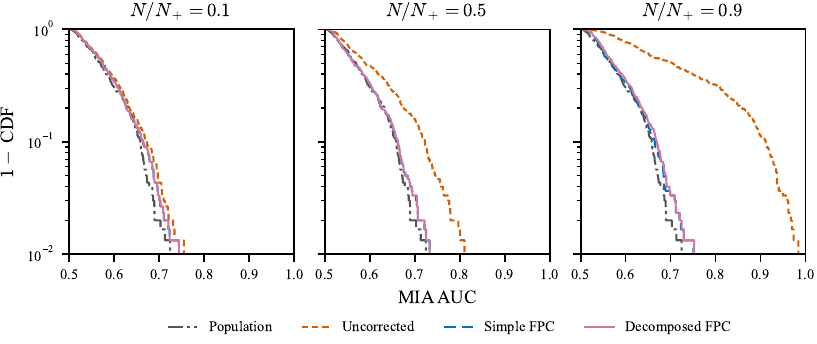}
    \caption{\textbf{Survival functions of per-sample MIA AUC demonstrating finite population bias for PatchCamelyon dataset}. We train 2048 ViT-B-16 models with final layer fine-tuned on PatchCamelyon dataset for varying $N/N_+$ (with $N=1000$). Per-sample MIA AUCs are pooled across 3 seeds and shown for \(N/N_+\in\{0.1,0.5,0.9\}\). We observe that finite population bias increasingly inflates estimated MIA vulnerability as $N/N_+$ grows. FPC (\cref{subsec:FPC}) substantially restores the population-level MIA vulnerability estimates.}
    \label{fig:adam_survival_plots}
\end{figure}

\begin{figure}[htbp]
    \centering
    \includegraphics[width=\columnwidth]{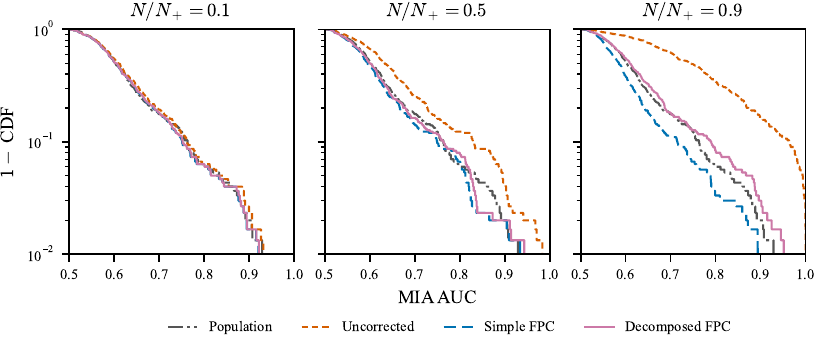}
    \caption{\textbf{Survival functions of per-sample MIA AUC demonstrating finite population bias for Adult dataset}. We train 2048 TabPFN models fine-tuned on Adult dataset for varying $N/N_+$ (with $N=1000$). Per-sample MIA AUCs are pooled across 3 seeds and shown for \(N/N_+\in\{0.1,0.5,0.9\}\). We observe that finite population bias increasingly inflates estimated MIA vulnerability as $N/N_+$ grows. Here, applying simple FPC (\cref{subsec:FPC}) tend to underestimate mean AUC w.r.t the population reference whereas whereas decomposed FPC (\cref{sec:variance-decom}) yields estimates closer to population-level MIA vulnerability estimates.}
    \label{fig:tabpfn_survival_plots}
\end{figure}

\begin{figure}[htbp]
    \centering
    \includegraphics[width=\linewidth]{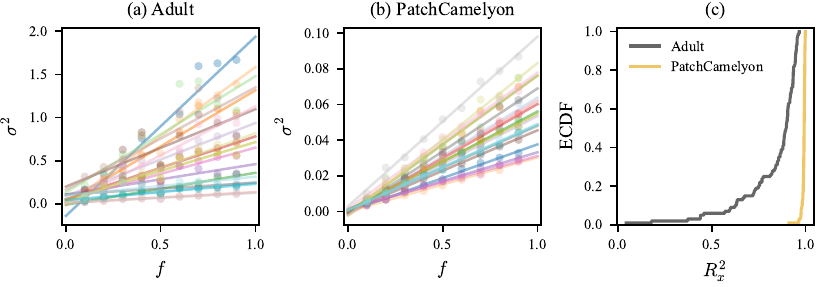}
    \caption{\textbf{Variance Decomposition for Adult and PatchCamelyon datasets for OUT distribution variances} (a-b) Observed variances and OLS fits $\sigma_x^2(f)=f\sigma_{\mathrm{pop},x}^2+\sigma_{\mathrm{rem},x}^2$ for 20 randomly selected target samples as computed per \cref{sec:variance-decom} for Adult and PatchCamelyon datasets.(c) Empirical CDF of per-target $R_x^2$ across all 100 evaluation targets using observed variances for all ratios $N/N_+\in\{0.1,\ldots,0.9\}$ for both the datasets. We train 2048 models per $N/N_+$, with $N=1000$, and use the collected statistics to estimate each target’s OUT variance.}
\label{fig:pc_var_decomposition_out}

\end{figure}

\subsection{Experimental Setup}
\label{subsec:experiment_setup}

\begin{itemize}[nosep,leftmargin=2em]
\item \textbf{Models} We use TabPFN~\citep{HollmannMPKKHSH25} for tabular binary classification. For image classification experiments, we fine-tune (a) final linear layer (Head) of
ViT-B-16 \citep{dosovitskiy_image_2020}/ ResNet50 \citep{kolesnikov_big_2019} model pretrained on ImageNet21K, (b) parameter-efficient FiLM adapters \citep{perez_film_2018} across ViT-B-16 model.

\item \textbf{Datasets} We use UCI's Adult~\citep{BeckerK96} dataset for tabular binary classification task to predict whether the Income is $<\$50$K or $\ge\$50$K. For the image classification experiments, we use CIFAR10 \citep{krizhevsky2009learning} and PatchCamelyon \citep{DBLP:conf/miccai/VeelingLWCW18} datasets.

\item \textbf{Training Paradigm} TabPFN leverages transformer-based in-context learning~\citep{DBLP:conf/nips/0001TLV22} to predict labels for the training samples. For image classification, we adopt a few-shot transfer learning setting, using a subset of training data to optimize the classification linear layer and, FiLM parameters. \cref{tab:model_hyperparameters} lists the hyperparameters used to train different models.
\end{itemize}

\begin{table}[htbp]
\centering
\caption{List of hyperparameters used in the experiments.}
\label{tab:model_hyperparameters}
\small
\setlength{\tabcolsep}{5pt}
\begin{tabular}{@{}lccc@{}}
\toprule
\textbf{Hyperparameter} & \textbf{Adult} & \textbf{PatchCamelyon} & \textbf{CIFAR10} \\
\midrule
Model & TabPFN & ViT-B-16 & ResNet50 \\
Adaptation & In-context & Linear & FiLM/Linear \\
\begin{tabular}[c]{@{}l@{}}
Training Set Size\\
\quad FPC (\Cref{subsec:FPC})\\
\quad Post-Processing (\Cref{subsec:post-processing})
\end{tabular}
&
\begin{tabular}[c]{@{}c@{}}\\1000\\10000\end{tabular}
&
\begin{tabular}[c]{@{}c@{}}\\1000\\--\end{tabular}
&
\begin{tabular}[c]{@{}c@{}}\\--\\500\end{tabular}
\\
Optimizer & -- & Adam & Adam \\
Epochs & -- & 500 & 50 \\
Batch size & -- & 128 & 64 \\
Learning rate & -- & $5\times10^{-3}$ & $8\times10^{-3}$ \\
Weight decay & -- & $10^{-3}$ & -- \\
\bottomrule
\end{tabular}
\end{table}

\begin{algorithm}[htbp]
\caption{Decomposed FPC for IN variances. The algorithm can be repurposed for OUT variances by using models excluding the target and
$f(r)=1-N/(N_+(r)-1)$}
\label{alg:decomposed-fpc}
\begin{algorithmic}[1]
\Require Training size $N$, targets $\mathcal{X}$,
ratios $r_a,r_b$, shadow models $M$, learning algorithm $\mathcal{A}$

\For{$r\in\{r_a,r_b\}$}
    \State Construct $D_{\mathrm{full}}^{(r)}\supseteq\mathcal{X}$
    with $N_+(r)=\operatorname{round}(N/r)$ records
    \For{$m=1,\ldots,M$}
        \State Sample $N$ records without replacement:
        $D^{(r,m)}\subseteq D_{\mathrm{full}}^{(r)}$
        \State Train $\mathcal{M}^{(r,m)}\gets \mathcal{A}(D^{(r,m)})$
        \State Record target statistics and membership signals for $\forall x \in \mathcal{X}$
    \EndFor
    \State Estimate $\widehat{\mu}_x(r)$ and
    $\widehat{\sigma}_{\mathrm{obs},x}^{2}(r)$
    using models whose training sets contain $x$
    \State $f(r)\gets 1-\dfrac{N-1}{N_+(r)-1}$
\EndFor

\State $\mathcal{V}_{\mathrm{corr}}\gets\varnothing$
\For{each target $x\in\mathcal{X}$}
    \State $\widehat{\sigma}_{\mathrm{pop},x}^{2}\gets
    \dfrac{\widehat{\sigma}_{\mathrm{obs},x}^{2}(r_a)
    -\widehat{\sigma}_{\mathrm{obs},x}^{2}(r_b)}
    {f(r_a)-f(r_b)}$
    \State $\widehat{\sigma}_{\mathrm{rem},x}^{2}\gets
    \dfrac{
    f(r_a)\widehat{\sigma}_{\mathrm{obs},x}^{2}(r_b)
    -f(r_b)\widehat{\sigma}_{\mathrm{obs},x}^{2}(r_a)
    }{
    f(r_a)-f(r_b)
    }$
    \State $V\gets\widehat{\sigma}_{\mathrm{pop},x}^{2}
    +\widehat{\sigma}_{\mathrm{rem},x}^{2}$
    \If{$V\leq 0$}
        \State Mark the decomposition invalid and \textbf{continue}
    \EndIf
    \State $B\gets\operatorname{clip}    (\widehat{\sigma}_{\mathrm{rem},x}^{2},0,V)$
    \For{$r\in\{r_a,r_b\}$}
        \State $\widehat{\sigma}_{\mathrm{corr},x}^{2}(r)\gets\widehat{\sigma}_{\mathrm{obs},x}^{2}(r)
        \dfrac{V}{f(r)(V-B)+B}$
        \State $\mathcal{V}_{\mathrm{corr}}\gets
    \mathcal{V}_{\mathrm{corr}}\cup
    \{(x,r,\widehat{\sigma}_{\mathrm{corr},x}^{2}(r))\}$
    \EndFor
\EndFor
\State \Return $\mathcal{V}_{\mathrm{corr}}$
\end{algorithmic}
\end{algorithm}

\begin{figure}[htbp]
    \centering\includegraphics[width=0.75\linewidth]{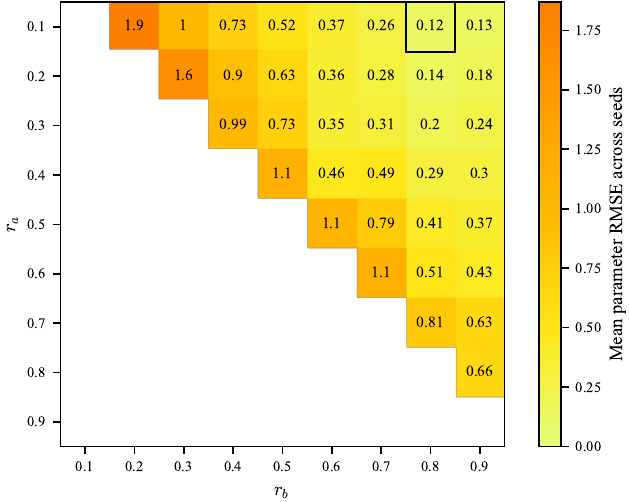}
    \caption{\textbf{Selecting  ratios for variance decomposition on Adult dataset.}
    Each cell of the heatmap compares parameter estimates obtained from a pair of sampling
    ratios $(r_a,r_b)$ with the corresponding all-ratio OLS estimates.
    RMSE is computed  for the 2 settings across all targets and then averaged over 3 repeats. Highlighted values indicate closest parameter alignment between decomposition using $(r_a,r_b)$ and all ratios.}
    \label{fig:vd_heatmap_adult}
\end{figure}

\begin{figure}[htbp]
    \centering
    \includegraphics[width=0.75\linewidth]{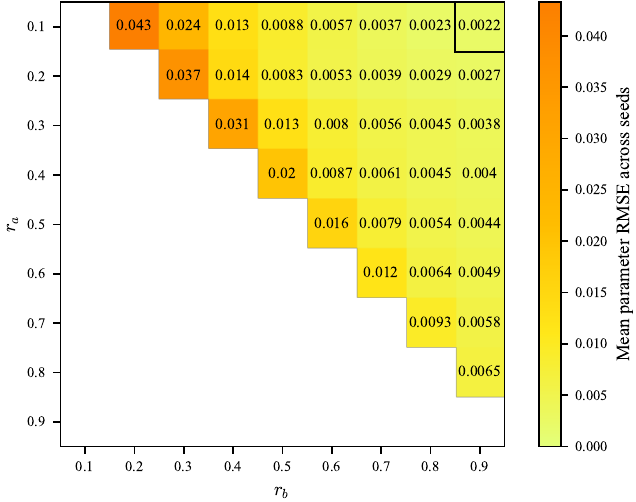}
    \caption{\textbf{Selecting ratios for variance decomposition on PatchCamelyon dataset.}
    Each cell of the heatmap compares parameter estimates obtained from a pair of sampling
    ratios $(r_a,r_b)$ with the corresponding all-ratio OLS estimates.
    RMSE is computed for the 2 settings across all targets and then averaged over 3 repeats. Highlighted values indicate closest parameter alignment between decomposition using $(r_a,r_b)$ and all ratios.}
    \label{fig:vd_heatmap_patch}
\end{figure}

\end{document}